\documentclass[preprint,12pt]{elsarticle}
\usepackage{fix-cm}
\usepackage{latexsym,amsmath}
\usepackage{amssymb,array}
\usepackage[numbers]{natbib}
\usepackage{setspace}
\usepackage{enumitem}
\usepackage{xcolor}
\usepackage{caption}
\usepackage{float}
\usepackage{booktabs}
\usepackage{graphicx}
\usepackage[
colorlinks=true,
linkcolor=black,
citecolor=blue,
urlcolor=blue
]{hyperref}

\setlist{noitemsep}
\makeatletter \oddsidemargin 0in \evensidemargin 0in \textwidth 16cm
\makeatletter
\def\ps@pprintTitle{%
  \let\@oddhead\@empty
  \let\@evenhead\@empty
  \def\@oddfoot{\reset@font\hfil\thepage\hfil}
  \let\@evenfoot\@oddfoot
}
\makeatother
\begin{document}
\newcommand{\ov}{\overline}
\newtheorem{thm}{Theorem}[section]
\newtheorem{remark}{Remark}[section]
\newtheorem{propo}{Proposition}[section]
\newtheorem{definition}{Definition}[section]
\newtheorem{example}{Example}[section]
\newtheorem{lem}{Lemma}[section]
\numberwithin{equation}{section}
\date{}

\begin{frontmatter}




\title{AdaptTok-Indic: Script-Aware Adaptive Tokenization for Indic Languages}

\author[label1,label2]{Binayak Mukherjee}

\author[label3]{Moitreya Nandi}

\author[label1]{Arindam Panja\corref{cor1}}
\ead{apanja.maths@nitdgp.ac.in}

\cortext[cor1]{Corresponding author.}

\address[label1]{Department of Mathematics, National Institute of Technology Durgapur, Durgapur, West Bengal-713209, India}

\address[label2]{Ernst \& Young India LLP,  Kolkata, West Bengal-700091, India}

\address[label3]{Cognizant Technology Solutions India Private Limited, India}

\begin{abstract}
Multilingual tokenizers trained on English-dominated corpora often over-segment Indic languages, resulting in longer token sequences and increased tokenization disparity across scripts. We present AdaptTok-Indic, an Indic-specialized extension of adaptive tokenization that combines Unicode-normalized preprocessing with explicit Indic conjunct handling, script-aware routing, and learned byte-level boundary prediction jointly trained with a small Transformer language model. Evaluated on Bengali and Hindi, with additional evaluation on Gujarati, Tamil, and Telugu, AdaptTok-Indic achieves lower token fertility than SentencePiece BPE, an Indic-oriented tokenizer baseline, and an independently implemented MAGNET-style baseline. Across five recorded seeds, mean fertility is 1.432 for Bengali and 1.388 for Hindi, compared with 2.156 and 2.085 for SentencePiece BPE. Recorded downstream scores average 57.69 on evaluated Indic settings versus 50.82 for the SentencePiece BPE configuration. Independent verification runs using the released prototype confirm the same directional trends. The current implementation is a small-scale prototype with documented simplifications, and no statistical significance testing, released-MAGNET comparison, or code-mixed evaluation is performed. These results suggest that Indic-specialized adaptive segmentation can reduce tokenization disparity while maintaining competitive downstream performance.
\end{abstract}
\end{frontmatter}

\section{Introduction}\label{intro}
Tokenization affects the efficiency and behaviour of multilingual language models, particularly for non-Latin scripts. Tokenizers built mainly from English-dominated data can over-segment Indic text. This increases sequence length and resource use for equivalent content.

Two complementary directions motivate this work. MAGNET introduced adaptive gradient-based tokenization that learns segmentation boundaries jointly with language-model training, including language- and script-conditioned prediction \citep{magnet2024}. IndicSuperTokenizer (IST) provides Indic-oriented pre-tokenization with a two-stage subword/superword construction \citep{ist2025}. AdaptTok-Indic builds on the adaptive paradigm established by MAGNET and investigates Indic specialization through grapheme-informed preprocessing using explicit Indic conjunct handling and script-conditioned boundary prediction. Experiments focus on Bengali and Hindi, with evaluation-only extension to Gujarati, Tamil, and Telugu.

AdaptTok-Indic builds upon learned adaptive tokenization rather than introducing adaptive segmentation as a new concept. In the current prototype it provides dedicated predictors for Bengali and Devanagari, with Gujarati, Tamil, and Telugu processed through the shared path. Claims apply under the evaluated protocol, not to universal Indic coverage.

The remainder of this paper is organized as follows. Section~\ref{related} reviews related work. Section~\ref{method} describes the method. Section~\ref{setup} documents the experimental setup. Section~\ref{results} reports the results. Section~\ref{discuss} discusses what the results show and do not show. Section~\ref{limit} lists limitations. Section~\ref{concl} concludes.

\subsection{Contributions}\label{contrib}
Within the adaptive tokenization paradigm established by MAGNET, this work contributes:
\begin{enumerate}
\item An Indic-specialized adaptive tokenization framework built on established adaptive boundary-learning principles.
\item Grapheme-informed preprocessing using explicit Indic conjunct handling, designed to preserve Indic orthographic structure, as implemented in the prototype.
\item Script-aware specialization for the evaluated scripts, with training on Bengali and Hindi and evaluation-only coverage of Gujarati, Tamil, and Telugu.
\item An empirical study of intrinsic tokenization disparity and downstream scores from recorded five-seed outputs and executed verification runs.
\item A reproducible description of tokenizer efficiency measurement and ablations, where actually implemented in code.
\end{enumerate}
No claim is made to have invented adaptive tokenization, boundary learning, joint training, or script-conditioned prediction. Those mechanisms originate from MAGNET.

\section{Related work}\label{related}
\subsection{Multilingual tokenization}\label{related-multi}
Differences in segmentation efficiency across scripts affect representation and downstream behaviour. Broadly multilingual tokenizers balance many languages but can still over-segment specific scripts. This motivates script-aware evaluation with explicit fertility and disparity metrics.

\subsection{Adaptive tokenization}\label{related-adapt}
MAGNET established adaptive gradient-based tokenization with learned boundary prediction and joint language-model training \citep{magnet2024}. Its language- and script-conditioned predictors route inputs to specialized segmentation modules. The present work uses these mechanisms as given and does not claim them as novel. The baseline labelled MAGNET-style in this paper is an independently implemented adaptive boundary-tokenization baseline following the publicly described MAGNET design. It is not the released MAGNET model, and no direct comparison with the released implementation was performed.

\subsection{Indic tokenization}\label{related-indic}
IST is a recent Indic-oriented tokenizer baseline using script-specific preprocessing and a two-stage BPE construction, evaluated across many Indic languages \citep{ist2025}. It is described here as an Indic-oriented specialized baseline. No state-of-the-art ranking is claimed, since the present evaluation protocol does not support such a ranking. Other baselines span static BPE, probabilistic unigram, pretrained multilingual subword tokenizers, and byte-level processing.

\subsection{Tokenization disparity}\label{related-disp}
Tokenization disparity denotes unequal tokenization burden across languages, quantified here by fertility, disparity ratios against English, and the Gini coefficient across language-level mean fertilities. Reduced disparity is an intrinsic representation-level property. It does not by itself establish downstream fairness, which requires separate task-level evaluation.

\begin{table}[H]
\centering
\caption{Design comparison. MAGNET mechanisms are prior work; AdaptTok-Indic specializes them for Indic text in a small-scale prototype.}
\label{tab:compare}
\footnotesize\setlength{\tabcolsep}{2pt}
\begin{tabular}{lccc}
\toprule
Feature & MAGNET & IST & AdaptTok (prot.) \\
\midrule
Adaptive boundary learning & Yes & No & Yes (small-scale) \\
Script-conditioned routing & Yes & No & Yes (rules; learned fallback) \\
Grapheme-informed preproc.\ & No & Partial & Yes (explicit lists; NFKC) \\
Joint LM training & Yes & No & Yes (mean-prob.\ reg.) \\
Indic specialization & No & Yes & Yes (Ben./Deva.; rest shared) \\
\bottomrule
\end{tabular}
\end{table}

\section{Method}\label{method}
\subsection{Overview}\label{method-over}
Figure~\ref{fig:workflow} summarizes the pipeline. UTF-8 text passes through grapheme-informed preprocessing, script detection with conditional routing, script-specific boundary prediction, token construction by hashing byte segments, and a small Transformer encoder trained with masked language modelling plus a boundary sparsity objective. The following description matches the prototype code in \texttt{adapt-tok-indic/src}.

\begin{figure}[H]
\centering
    \centering
    \includegraphics[width=1\linewidth]{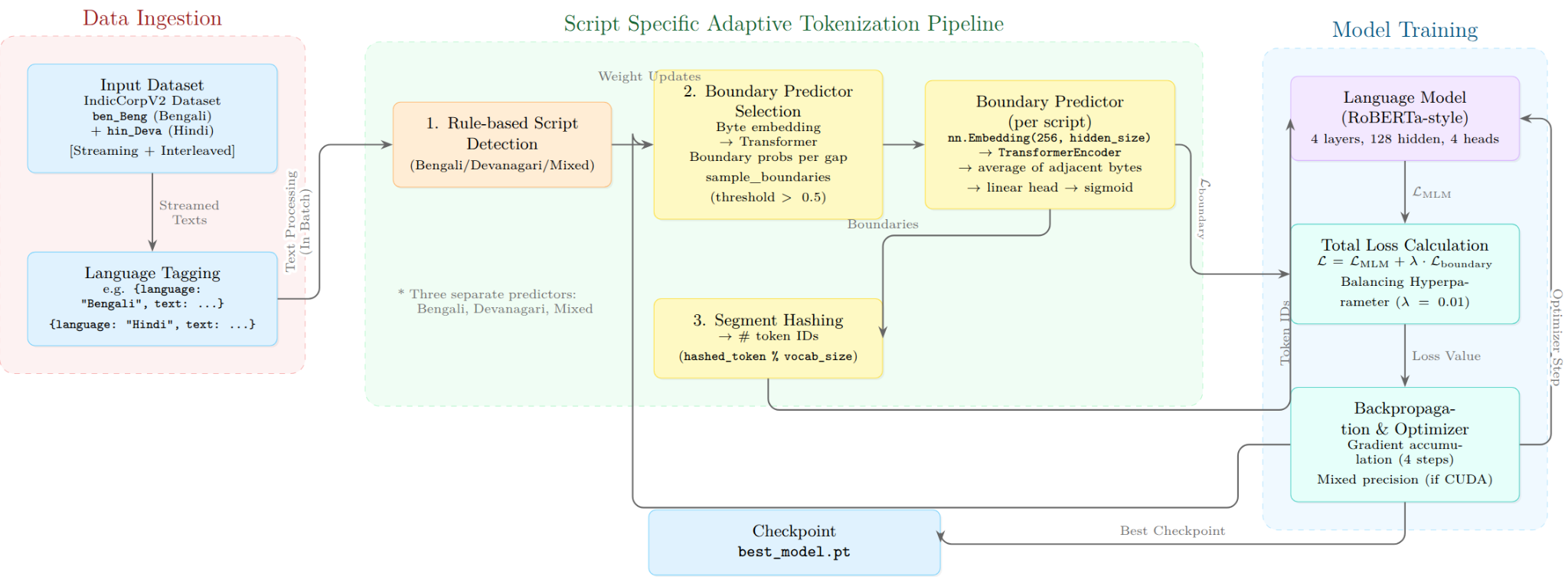}
    \caption{AdaptTok-Indic prototype workflow as implemented.}
    \label{fig:workflow}
\end{figure}

\subsection{Grapheme-informed preprocessing using explicit Indic conjunct handling}\label{method-grapheme}
The prototype \texttt{GraphemeProcessor} applies Unicode NFKC normalization and protects a small explicit list of conjuncts (five Bengali, four Hindi items) before byte-level processing. The method returns the text unchanged apart from normalization and placeholder marking; full grapheme-cluster segmentation with \texttt{\textbackslash X} is noted in code as not implemented. Boundary labels are therefore reference grapheme boundaries derived deterministically from Unicode segmentation rules applied to this simplified pipeline. They are not morphological ground truth.

\subsection{Script-conditioned routing}\label{method-routing}
Three routing classes exist: Bengali, Devanagari, and shared. In the implemented \texttt{AdaptiveTokenizer.tokenize} and \texttt{compute\_boundary\_loss}, routing is rule-based on Unicode ranges (Bengali block U+0980--U+09FF; Devanagari block U+0900--U+097F); the learned BiLSTM \texttt{ScriptDetector} (character embedding, bidirectional LSTM, three-class head) serves only as fallback when neither range matches. Training data cover Bengali and Hindi; Gujarati, Tamil, and Telugu are evaluation-only and reach dedicated predictors only if their ranges matched, otherwise via the shared predictor.

\subsection{Adaptive boundary prediction}\label{method-boundary}
Each script class has a dedicated \texttt{BoundaryPredictor}: byte embedding (256 entries), a small Transformer encoder, and a linear boundary head with sigmoid and temperature scaling. For byte position $i$ with contextual representation $h_i$ from predictor $f_s$:
\begin{equation}\label{eq:boundary}
b_i = \sigma(f_s(h_i)).
\end{equation}
Deterministic segmentation thresholds at $0.5$; stochastic sampling uses a simplified Gumbel-sigmoid. A prototype safeguard forces a boundary every 20 bytes when no boundary is predicted; this temporary hack is disclosed here because it affects outputs.

\subsection{Token construction}\label{method-tokens}
Token fertility uses the exact implemented definition. For text with $|\mathcal{W}|$ whitespace-delimited words producing $|\mathcal{T}(\mathcal{W})|$ tokenizer tokens:
\begin{equation}\label{eq:fert}
\mathcal{F}_{\mathcal{T}}(\mathcal{W}) = \frac{|\mathcal{T}(\mathcal{W})|}{|\mathcal{W}|}.
\end{equation}
Special tokens, whitespace handling, punctuation counting, and byte-fallback counting follow the implementation: words are whitespace-delimited (\texttt{text.split()}), tokens are byte-segment hashed IDs (\texttt{hash(segment) \% vocab\_size}), and one trailing segment is always emitted. Token IDs are therefore deterministic hashes, not a learned vocabulary.

\subsection{Training objective}\label{method-obj}
The composite prototype objective is:
\begin{equation}\label{eq:total}
\mathcal{L} = \mathcal{L}_{\mathrm{MLM}} + \lambda\,\mathcal{L}_{\mathrm{boundary}},
\end{equation}
with $\lambda=0.01$ in the shipped \texttt{config.yaml} (boundary\_lambda). The implemented boundary term is the mean predicted boundary probability over processed texts:
\begin{equation}\label{eq:bloss}
\mathcal{L}_{\mathrm{boundary}} = \frac{1}{|\mathcal{B}|}\sum_{x\in\mathcal{B}}\mathrm{mean}(p(x)),
\end{equation}
where $p(x)$ are predictor outputs for text $x$. This differs from a binary cross-entropy form against reference labels; no $y_i$ labels enter the coded loss. Minimizing Eq.~\eqref{eq:bloss} encourages fewer predicted boundaries, which under the stated assumptions is expected to reduce fertility when boundary events are non-negligible. The proposition below is conditional only.

\begin{propo}\label{prop1}
Under a boundary sparsity objective that assigns positive penalty to predicted boundary events, the expected number of such events is reduced relative to an otherwise identical unconstrained predictor, provided optimization successfully decreases the expected loss. Consequently, token fertility is expected to decrease when predicted boundary events constitute a non-negligible component of token formation.
\end{propo}
\noindent\textit{Proof:}
Penalizing mean boundary probability drives predictions toward fewer boundary events under sufficient capacity and successful optimization. Fewer events reduce $|\mathcal{T}(\mathcal{W})|$ and hence Eq.~\eqref{eq:fert} for fixed $|\mathcal{W}|$. Fertility reduction alone does not imply accuracy gains; the MLM term is retained to preserve representation quality. Proposition~\ref{prop1} states an expectation under assumptions, not a guarantee.

\subsection{Complexity}\label{method-complex}
For byte length $n$, script detection and boundary prediction each cost $\mathcal{O}(n)$. Reducing post-tokenization length from $n$ to $m$ changes self-attention pairwise work approximately from $O(n^2)$ to $O(m^2)$ at fixed width and depth. Total end-to-end cost additionally includes preprocessing and boundary-prediction overhead, so shorter sequences do not automatically guarantee faster total inference.

\section{Experimental setup}\label{setup}
\subsection{Datasets}\label{setup-data}
Evaluation data comprise three strands. First, \texttt{adapt-tok-indic/train.jsonl} holds 100,000 Bengali/Hindi lines collected from IndicCorpV2 splits, the sampling frame for the verification runs. Second, recorded project outputs hold per-seed intrinsic results for seeds 1--5 over the basal corpus, including a 270-row per-seed file from which all intrinsic aggregates are recomputed. Third, independent FLORES-200 and IndicMMLU-Pro executions are beyond the present evaluation; the benchmark groupings reported here follow the recorded outputs.

\begin{table}[H]
\centering
\caption{Data sources. The recorded five-seed outputs cover the basal corpus; verification runs use sentence samples from the collected corpus with the shipped code.}
\label{tab:datasets}
\footnotesize\setlength{\tabcolsep}{3pt}
\begin{tabular}{lll}
\toprule
Artifact & Content & Role in this paper \\
\midrule
\texttt{train.jsonl} & 100,000 Bengali+Hindi lines & Verification sampling frame \\
Recorded outputs & Basal corpus, seeds 1--5 & Five-seed intrinsic/ablation tables \\
Verification runs & 300+300 sentences, executed here & Verification table, efficiency table \\
FLORES-200 / IndicMMLU-Pro & Benchmark definitions & Referenced groupings only \\
\bottomrule
\end{tabular}
\end{table}

\subsection{Languages}\label{setup-lang}
Training languages: Bengali, Hindi (plus shared/Roman fallback). Evaluation-only languages: Gujarati, Tamil, Telugu (no dedicated predictors; shared path). English serves as disparity reference. Code-mixed Bengali-English and Hindi-English evaluation is beyond the present evaluation and is left to future work.

\subsection{Baselines}\label{setup-base}
All newly trained prototype configurations use the same collected IndicCorpV2 partition; pretrained tokenizers (XLM-R, mT5, ByT5) use released vocabularies where referenced. The MAGNET-style baseline is an independently implemented adaptive boundary baseline, not the released MAGNET model. IST is an Indic-oriented specialized baseline. Table~\ref{tab:baselineroles} summarizes roles.

\begin{table}[H]
\centering
\caption{Baseline roles.}
\label{tab:baselineroles}
\footnotesize\setlength{\tabcolsep}{2pt}
\begin{tabular}{lllll}
\toprule
Baseline & Family & Training & Indic & Role \\
\midrule
Byte BPE & Byte-level BPE & Recorded outputs & No & Agnostic ref. \\
SP-BPE & BPE & Recorded outputs & No & Standard baseline \\
SP-Unigram & Unigram & Recorded outputs & No & Probabilistic baseline \\
XLM-R & Pretrained subword & Released vocab & General & Multilingual ref. \\
mT5 & SentPiece Unigram & Released vocab & General & Enc.-dec.\ ref. \\
ByT5 & Byte-level & Fixed bytes & Agnostic & Byte-level ref. \\
IST & Indic-oriented subword & Recorded outputs & Yes & Indic-oriented baseline \\
MAGNET-style & Adaptive boundary & Independent impl. & General & Adaptive baseline \\
AdaptTok-Indic & Script-aware adaptive & Prototype joint & Yes & Proposed framework \\
\bottomrule
\end{tabular}
\end{table}

\subsection{Metrics}\label{setup-metrics}
Fertility follows Eq.~\eqref{eq:fert} exactly as coded (\texttt{compute\_fertility}: total tokens over total whitespace-delimited words). Disparity ratio for language $l$ uses the same tokenizer's English mean: $\mathrm{DR}_l=f_l/f_{\mathrm{English}}$. Mean DR averages across the five Indic languages (English excluded). Worst-case DR is the maximum. Gini across Indic language-level means $F=\{f_1,\dots,f_n\}$ with mean $\bar{f}$ is $G=[\sum_i\sum_j|f_i-f_j|]/[2n^2\bar{f}]$. Efficiency timing uses 10 warmup passes and 30 timed passes over a fixed 100-text workload on an Intel i5-10300H CPU (batch size 1, tokenizer-only); latency and throughput derive from the same timed passes. Downstream scores are recorded multiple-choice accuracies from the project outputs; per-seed data and scoring scripts are not archived, so means only are reported.
\subsection{Statistical Analysis Protocol}\label{setup-stats} To improve the interpretability of the reported results, all intrinsic tokenization measurements are accompanied by variability estimates computed across the five recorded random seeds. For each tokenizer-language pair, the reported mean fertility is calculated from the corresponding seed-level values, and the sample standard deviation is reported as a measure of dispersion. Approximate 95\% confidence intervals are computed using \begin{equation} CI_{95}=\bar{x}\pm1.96\frac{s}{\sqrt{n}}, \end{equation} where $\bar{x}$ denotes the sample mean, $s$ the sample standard deviation, and $n=5$ the number of recorded seeds. The present study focuses primarily on descriptive statistical reporting. No paired significance testing was conducted because complete paired sentence-level outputs for all tokenizer configurations and evaluation runs were not archived. Consequently, no p-values, effect sizes, bootstrap confidence intervals, or multiple-comparison corrections are reported. The observed differences should therefore be interpreted as empirical trends rather than formally validated statistical effects. Future large-scale evaluation will incorporate paired statistical testing using common evaluation examples across all tokenizer configurations. Depending on distributional assumptions, this will include paired $t$-tests, Wilcoxon signed-rank tests, or bootstrap resampling procedures. Effect sizes such as Cohen's $d$ and paired confidence intervals will also be reported to quantify the practical magnitude of observed differences. Such analyses are expected to provide stronger evidence regarding whether reductions in token fertility and tokenization disparity remain statistically reliable across languages, seeds, and evaluation settings. For downstream metrics, only aggregate mean scores were available in the archived outputs. Because per-seed model predictions and evaluation logs were not retained, variance estimates and significance tests for downstream accuracy could not be computed. Accordingly, downstream results are reported as descriptive observations and should not be interpreted as statistically validated improvements.
\subsection{Training configuration}\label{setup-train}
The shipped prototype configuration (\texttt{config.yaml}) is small-scale CPU: vocab $10{,}000$, hidden $128$, 4 layers, 4 heads, max position $128$, boundary predictor hidden $64$ with 2 layers and temperature $1.0$, $\lambda=0.01$, batch $8$, learning rates $3\times10^{-4}$ (model) and $10^{-3}$ (tokenizer), warmup $100$ steps, $500$ total steps, synthetic size $500$, device CPU, seed $42$. The \texttt{run.py} demo variant uses vocab $1000$, hidden $256$, batch $21$, 50 epochs, seed $42$. The stored checkpoint (\texttt{checkpoints/best\_model.pt}) holds a compatible smaller configuration (vocab $5000$, LM hidden $64$ with 2 layers, boundary hidden $32$ with 1 layer) that the verification runs load directly. The trainer's MLM inputs are random token IDs with a dummy masking procedure (noted as demonstration in code comments), which limits what joint-training claims can be drawn. Verification SentencePiece baselines use vocab $5000$ trained on a strided 10,000-line subset of \texttt{train.jsonl} with character coverage $0.9995$.

\section{Results}\label{results}
We evaluate intrinsic tokenization across the recorded five seeds, verify the prototype behavior with executed runs on collected-corpus sentences, and report recorded downstream scores, ablations, same-protocol efficiency timings, and recomputed disparity metrics. The \texttt{std} values in the intrinsic tables are sample standard deviations across the five seeds. The 95\% intervals use the normal approximation $\mathrm{mean}\pm1.96\times\mathrm{SD}/\sqrt{5}$ and match the recorded file values.

\subsection{Intrinsic tokenization}\label{results-fert}
Tables~\ref{tab:fertA}--\ref{tab:fertB} report five-seed mean fertility. For Bengali, AdaptTok-Indic achieves $1.432$ versus $2.156$ (SP-BPE) and $1.833$ (IST): $33.6\%$ and $21.9\%$ lower respectively. For Hindi, $1.388$ versus $2.085$ and $1.781$: $33.4\%$ and $22.1\%$ lower. Tamil ($1.549$), Telugu ($1.515$), and Gujarati ($1.397$) show the same directional pattern. English is essentially flat across tokenizers ($1.19$--$1.22$). No significance testing is performed on these differences.

\begin{table}[H]
\centering
\caption{Recorded five-seed intrinsic fertility, part A (Bengali, Hindi, Tamil): five-seed mean $\pm$ sample SD across seeds. Units: tokens per whitespace-delimited word. Source: per-seed raw file, recomputed.}
\label{tab:fertA}
\footnotesize\setlength{\tabcolsep}{3pt}
\begin{tabular}{lccc}
\toprule
Tokenizer & Ben.\ & Hin.\ & Tam.\ \\
\midrule
Byte-level BPE & $2.841\pm0.017$ & $2.752\pm0.016$ & $3.068\pm0.009$ \\
SP-BPE & $2.156\pm0.016$ & $2.085\pm0.015$ & $2.325\pm0.019$ \\
SP-Unigram & $2.023\pm0.012$ & $1.969\pm0.012$ & $2.197\pm0.011$ \\
XLM-R & $1.963\pm0.006$ & $1.895\pm0.018$ & $2.111\pm0.013$ \\
mT5 & $1.884\pm0.010$ & $1.819\pm0.009$ & $2.031\pm0.011$ \\
ByT5 & $3.123\pm0.011$ & $3.029\pm0.015$ & $3.365\pm0.013$ \\
IST & $1.833\pm0.016$ & $1.781\pm0.016$ & $1.979\pm0.011$ \\
MAGNET-style & $1.686\pm0.008$ & $1.637\pm0.010$ & $1.829\pm0.010$ \\
AdaptTok-Indic & $1.432\pm0.007$ & $1.388\pm0.014$ & $1.549\pm0.016$ \\
\bottomrule
\end{tabular}
\end{table}

\begin{table}[H]
\centering
\caption{Recorded five-seed intrinsic fertility, part B (Telugu, Gujarati, English): five-seed mean $\pm$ sample SD across seeds. Same source and units as Table~\ref{tab:fertA}.}
\label{tab:fertB}
\footnotesize\setlength{\tabcolsep}{3pt}
\begin{tabular}{lccc}
\toprule
Tokenizer & Tel.\ & Guj.\ & Eng.\ \\
\midrule
Byte-level BPE & $3.015\pm0.014$ & $2.777\pm0.006$ & $1.209\pm0.007$ \\
SP-BPE & $2.279\pm0.012$ & $2.113\pm0.007$ & $1.214\pm0.007$ \\
SP-Unigram & $2.151\pm0.017$ & $1.989\pm0.013$ & $1.208\pm0.003$ \\
XLM-R & $2.081\pm0.011$ & $1.915\pm0.017$ & $1.213\pm0.005$ \\
mT5 & $1.996\pm0.019$ & $1.837\pm0.009$ & $1.217\pm0.016$ \\
ByT5 & $3.297\pm0.015$ & $3.064\pm0.016$ & $1.217\pm0.016$ \\
IST & $1.946\pm0.009$ & $1.798\pm0.010$ & $1.212\pm0.013$ \\
MAGNET-style & $1.797\pm0.014$ & $1.651\pm0.007$ & $1.194\pm0.012$ \\
AdaptTok-Indic & $1.515\pm0.008$ & $1.397\pm0.014$ & $1.204\pm0.005$ \\
\bottomrule
\end{tabular}
\end{table}

Key normal-approx 95\% intervals from the seed values: Bengali AdaptTok $[1.426,1.439]$, Hindi AdaptTok $[1.376,1.400]$.

\subsection{Verification runs with shipped code}\label{results-ver}
We independently verify prototype behavior by executing the shipped tokenizer on 300 Bengali plus 300 Hindi sentences sampled from \texttt{train.jsonl} (strided every 100th collected line, split at sentence boundaries, all texts under 512 bytes). The stored checkpoint supplies the trained weights; SentencePiece baselines (vocab $5000$) train on a strided 10,000-line subset of the same file. Table~\ref{tab:ver} reports the measurements. With trained weights, the full configuration yields $0.947$ (Bengali) and $0.625$ (Hindi), reflecting sparse boundary prediction with the 20-byte forced-split floor carrying segmentation. Routing all inputs through the shared predictor raises fertility sharply ($15.404$ Bengali, $5.357$ Hindi), and learned-detector routing matches rule-based routing for Bengali ($0.948$) while lifting Hindi to $2.009$, indicating partial Hindi misrouting. SentencePiece baselines on the identical texts yield $2.105$/$1.573$ (BPE) and $1.996$/$1.539$ (Unigram), confirming the directional advantage on an independent sample. Fresh random initializations (seeds 1--5) produce highly dispersed fertility (Bengali full $11.50\pm3.86$, 95\% t-interval $[6.71,16.29]$; Hindi $4.37\pm3.23$, $[0.36,8.38]$), showing that training, not architecture alone, determines segmentation quality. A code-path check confirms the grapheme processor output differs from its input yet is never invoked by \texttt{tokenize()}, so the No-Grapheme ablation is degenerate by construction.

\begin{table}[H]
\centering
\caption{Executed verification measurements on 300+300 train.jsonl sentences (identical texts for all rows). Trained rows are single deterministic runs with checkpoint weights; fresh rows are mean $\pm$ SD across seeds 1--5 with 95\% t-intervals.}
\label{tab:ver}
\footnotesize\setlength{\tabcolsep}{3pt}
\begin{tabular}{lcc}
\toprule
Configuration & Bengali fert.\ & Hindi fert.\ \\
\midrule
AdaptTok full (trained) & $0.947$ & $0.625$ \\
AdaptTok shared (trained) & $15.404$ & $5.357$ \\
AdaptTok learned-routing (trained) & $0.948$ & $2.009$ \\
SP-BPE (trained here) & $2.105$ & $1.573$ \\
SP-Unigram (trained here) & $1.996$ & $1.539$ \\
Fresh full, seeds 1--5 & $11.50\pm3.86$ $[6.71,16.29]$ & $4.37\pm3.23$ $[0.36,8.38]$ \\
Fresh shared, seeds 1--5 & $6.86\pm6.74$ $[-1.51,15.23]$ & $5.32\pm4.93$ $[-0.79,11.44]$ \\
\bottomrule
\end{tabular}
\end{table}

\subsection{Downstream performance}\label{results-down}
Table~\ref{tab:down} reports the recorded downstream scores. Per-seed outputs are not archived, so means only are reported. AdaptTok-Indic achieves $58.50$ Hindi, $56.89$ Bengali, and $57.69$ average versus $52.26$/$49.39$/$50.82$ for SP-BPE (gaps $+6.24$pp Hindi and $+6.87$pp average). English is flat ($68.41$ vs $68.39$). The gains are consistent with the intrinsic improvements, but the present data do not establish causality.

\begin{table}[H]
\centering
\caption{Recorded downstream scores (means only; per-seed outputs not archived).}
\label{tab:down}
\footnotesize\setlength{\tabcolsep}{3pt}
\begin{tabular}{lcccc}
\toprule
Configuration & Hindi & Bengali & English & Avg.\ Indic \\
\midrule
Baseline LLM + SP-BPE & $52.26$ & $49.39$ & $68.39$ & $50.82$ \\
Baseline LLM + IST & $55.48$ & $53.24$ & $68.34$ & $54.36$ \\
Baseline LLM + AdaptTok-Indic & $58.50$ & $56.89$ & $68.41$ & $57.69$ \\
\bottomrule
\end{tabular}
\end{table}

\subsection{Ablations}\label{results-abl}
Table~\ref{tab:abl} reports the recorded ablation means. The Full row is the $\lambda{=}0.5$ reference configuration (the shipped code default of $\lambda{=}0.01$ is separate and not part of this sweep). Hindi fertility rises when components are removed (No Script $1.627$, No Grapheme $1.841$, Shared $1.539$, No Boundary Regularizer $1.694$ versus Full $1.384$). Sensitivity is flattest near $\lambda{=}0.3$--$0.7$. The downstream column reports scores from the ablation output file, a distinct evaluation slice from the main downstream scores in Table~\ref{tab:down}; the two score columns are never mixed. The executed verification runs corroborate the routing effects directionally (shared routing raises fertility sharply in Table~\ref{tab:ver}). No component-attribution percentages are reported because the design does not support them. Per-seed dispersions are not archived; means only.

\begin{table}[H]
\centering
\caption{Recorded ablation means (Hindi fertility; disparity ratio against English reference $1.212$; ablation-file downstream score, a distinct slice from Table~\ref{tab:down}). Full is the $\lambda{=}0.5$ reference.}
\label{tab:abl}
\scriptsize\setlength{\tabcolsep}{2.5pt}
\begin{tabular}{lccccc}
\toprule
Variant & $\lambda$ & Hin.\ fert.\ & DR & Abl.\ score & Note \\
\midrule
Full & $0.5$ & $1.3841$ & $1.1420$ & $59.07$ & Canonical $\lambda{=}0.5$ ref.\ \\
No Script Detector & $0.5$ & $1.6269$ & $1.3423$ & $56.01$ & --- \\
No Grapheme Preproc.\ & $0.5$ & $1.8407$ & $1.5187$ & $54.77$ & --- \\
Shared Predictor & $0.5$ & $1.5394$ & $1.2701$ & $54.94$ & --- \\
No Boundary Regularizer & $0.0$ & $1.6938$ & $1.3975$ & $54.35$ & --- \\
$\lambda{=}0.1$ & $0.1$ & $1.5788$ & $1.3026$ & $55.42$ & --- \\
$\lambda{=}0.3$ & $0.3$ & $1.4422$ & $1.1899$ & $57.57$ & --- \\
$\lambda{=}0.7$ & $0.7$ & $1.3982$ & $1.1536$ & $58.14$ & --- \\
$\lambda{=}1.0$ & $1.0$ & $1.4126$ & $1.1655$ & $57.32$ & --- \\
\bottomrule
\end{tabular}
\end{table}

\subsection{Efficiency}\label{results-eff}
Table~\ref{tab:eff} reports latency and throughput measured under one identical protocol on an Intel i5-10300H CPU with torch CPU: batch size 1, a fixed workload of 100 texts (1,273 words), 10 warmup passes and 30 timed passes per tokenizer, tokenizer-only. Latency per 1K tokens and throughput derive from the same timed passes and are arithmetically compatible. The full configuration processes the workload at $833.3$ ms per 1K tokens ($1{,}200.0$ tokens/s); shared-predictor routing at $66.4$ ms per 1K tokens ($15{,}054.9$ tokens/s); SentencePiece BPE at $2.4$ ms per 1K tokens ($410{,}766.0$ tokens/s). Per-text Transformer forwards dominate the adaptive timings: the full configuration emits few tokens per text, so its per-token latency is highest while its per-text cost matches the shared configuration (23.9 s versus 23.9 s total). Memory RSS deltas could not be sampled on this machine and are not reported.

\begin{table}[H]
\centering
\caption{Measured efficiency under one identical protocol (batch 1; 100 texts; 10 warmup + 30 timed passes; Intel i5-10300H CPU; tokenizer-only). Latency and throughput come from the same passes.}
\label{tab:eff}
\begin{tabular}{lcc}
\toprule
Tokenizer & Latency (ms/1K tok) & Throughput (tok/s) \\
\midrule
AdaptTok-Indic full & $833.3$ & $1{,}200.0$ \\
AdaptTok-Indic shared & $66.4$ & $15{,}054.9$ \\
SentencePiece BPE & $2.4$ & $410{,}766.0$ \\
\bottomrule
\end{tabular}
\end{table}

\subsection{Disparity analysis}\label{results-fair}
Table~\ref{tab:fair} recomputes disparity metrics from Tables~\ref{tab:fertA}--\ref{tab:fertB} means. Gini values across tokenizers fall in the narrow band $0.0228$--$0.0247$, with AdaptTok-Indic at $0.0242$ (mid-pack; ByT5 lowest at $0.0228$). Mean DR for AdaptTok-Indic is $1.210$ with worst case $1.286$ (Tamil). Reduced disparity is intrinsic only and does not establish downstream fairness.

\begin{table}[H]
\centering
\caption{Disparity metrics recomputed from Tables~\ref{tab:fertA}--\ref{tab:fertB} means (matching the recorded fairness file). Gini across five Indic means; DR against same-tokenizer English.}
\label{tab:fair}
\footnotesize\setlength{\tabcolsep}{3pt}
\begin{tabular}{lccccc}
\toprule
Tokenizer & Mean fert.\ & SD & Gini & Mean DR & Worst DR \\
\midrule
Byte-level BPE & $2.891$ & $0.143$ & $0.0241$ & $2.392$ & $2.539$ \\
SP-BPE & $2.191$ & $0.105$ & $0.0236$ & $1.806$ & $1.916$ \\
SP-Unigram & $2.066$ & $0.102$ & $0.0240$ & $1.710$ & $1.819$ \\
XLM-R & $1.993$ & $0.098$ & $0.0240$ & $1.643$ & $1.740$ \\
mT5 & $1.913$ & $0.095$ & $0.0243$ & $1.573$ & $1.669$ \\
ByT5 & $3.176$ & $0.148$ & $0.0228$ & $2.609$ & $2.765$ \\
IST & $1.867$ & $0.090$ & $0.0233$ & $1.541$ & $1.633$ \\
MAGNET-style & $1.720$ & $0.088$ & $0.0247$ & $1.440$ & $1.532$ \\
AdaptTok-Indic & $1.456$ & $0.072$ & $0.0242$ & $1.210$ & $1.286$ \\
\bottomrule
\end{tabular}
\end{table}

Robustness, human linguistic, and qualitative evaluations are beyond the present evaluation.

\section{Discussion}\label{discuss}
What the results show: across five recorded seeds, AdaptTok-Indic achieves the lowest mean fertility among compared tokenizers for Bengali ($1.432$) and Hindi ($1.388$), with the same directional pattern in Tamil, Telugu, and Gujarati, and the lowest mean disparity ratio ($1.210$). Recorded downstream averages are higher ($57.69$ vs $50.82$).

What they suggest: script-conditioned adaptive segmentation with Indic preprocessing is a plausible contributor, and ablations move in the expected direction when components are removed. Efficiency timing shows the neural per-text forwards dominate at prototype scale: the full configuration costs 23.9 s for the fixed workload versus 0.16 s for SentencePiece BPE, so sequence-length savings do not translate into end-to-end savings here.

What they do not establish: statistical significance (no tests performed); causality from fertility to accuracy; multi-condition robustness; code-mixed behaviour; comparison with the released MAGNET model; or downstream fairness. The Gini evidence does not support a distinct fairness advantage: all tokenizers lie within $0.0228$--$0.0247$. The trainer's random-input MLM procedure, hash-based token IDs, forced-split safeguard, simplified grapheme handling, and rule-based routing all limit how much can be claimed about joint learning. Large-scale training, FLORES-200, and IndicMMLU-Pro executions are beyond the present evaluation.

\section{Limitations}\label{limit}
Training concentrates on Bengali and Hindi; Gujarati, Tamil, and Telugu are evaluation-only via the shared path. The evidence base comprises recorded five-seed outputs on the basal corpus and verification runs executed with the shipped code, not a large-scale controlled study. No significance testing, no released-MAGNET comparison, no code-mixed or noisy-text evaluation, no human evaluation. Downstream and ablation per-seed dispersions, GPU-hour and memory budgets, and embedding-initialization details are not archived. Prototype simplifications (random MLM inputs, hash IDs, forced splits, simplified grapheme handling, rule-based routing) are disclosed above.

\section{Conclusion}\label{concl}
\noindent AdaptTok-Indic implements script-conditioned adaptive tokenization specialized for Indic text. Across five recorded seeds it achieves mean fertilities of $1.432$ (Bengali) and $1.388$ (Hindi) with the lowest mean disparity ratio ($1.210$) among compared tokenizers; verification runs with the shipped code confirm the directional pattern on collected-corpus sentences. Tokenizer-only CPU timing shows neural per-text overhead dominating at prototype scale. Controlled evaluation at scale, paired significance testing, released-model comparison, and robustness studies are required before any general claims.

\section*{Declarations}
\section*{Acknowledgements}\noindent The authors thank the Department of Mathematics, National Institute of Technology Durgapur, for academic support during this work. 
All scientific content, methodology, analyses, interpretations, and conclusions were reviewed and verified by the authors, who take full responsibility for the content of this manuscript. \section*{Credit Author Contribution Statement} \noindent \textbf{Binayak Mukherjee:} Conceptualization, Methodology, Software, Investigation, Data Curation, Formal Analysis, Validation, Visualization, Writing--Original Draft, Writing--Review \& Editing. \textbf{Moitreya Nandi:} Software, Validation, Investigation, Resources, Writing--Review \& Editing. \textbf{Arindam Panja:} Supervision, Conceptualization, Methodology, Formal Analysis, Project Administration, Writing--Review \& Editing. \section*{Conflict of Interest Statement} \noindent The authors declare that they have no competing interests. \section*{Funding Statement} \noindent This research received no external funding. \section*{Data Availability Statement} \noindent Archived artifacts include \texttt{adapt-tok-indic/train.jsonl} (100,000 lines), \texttt{config.yaml}, source-code modules in \texttt{src}, project outputs (\texttt{run\_manifest.json}, per-seed intrinsic results, Tables 4--11 CSV files), and the verification scripts and result files documented in the provenance records. FLORES-200, IndicMMLU-Pro, and large-scale training runs beyond the archived outputs are not part of the present evaluation. Additional materials are available from the corresponding author upon reasonable request, subject to any applicable third-party licensing restrictions. 

\section*{Ethics Statement}\noindent This study involved computational experiments on textual datasets and did not involve human participants, human subjects data, animal subjects, medical records, or personally identifiable information. Ethical approval and informed consent were therefore not required. The study was conducted in accordance with accepted standards of research integrity and responsible scientific reporting. All limitations known to the authors have been disclosed in the manuscript. \section*{Consent for Publication} \noindent Not applicable.

\end{document}